\documentclass[11pt,a4paper]{articleLH}

\providecommand{\N}{{\mathbb N}}
\providecommand{\R}{{\mathbb R}}
\providecommand{\id}{I}
\providecommand{\Sigmahat}{\widehat{\Sigma}}
\providecommand{\Sigmahatlambda}[1][\lambda]{\widehat{\Sigma}_{#1}}
\providecommand{\Sigmalambda}[1][\lambda]{\Sigma_{#1}}
\providecommand{\ridgeest}[1][\lambda]{\hat{\beta}^{\mathrm{RR}}_{#1}}
\providecommand{\ridgegdest}[2][\eta]{\hat{\beta}^{\mathrm{GD}}_{\lambda,{#1},{#2}}}
\providecommand{\pengradflowest}[2][\lambda]{\hat{\beta}^{\mathrm{GF}}_{{#1},{#2}}}
\providecommand{\ridgecgest}[2][\lambda]{\hat{\beta}^{\mathrm{CG}}_{{#1},{#2}}}
\providecommand{\ridgeapproxerrorCG}[2][\lambda]{A^{\mathrm{CG}}_{{#1},\beta_{\lambda},{#2}}}
\providecommand{\ridgestocherrorCG}[2][\lambda]{S^{\mathrm{CG}}_{{#1},{#2}}}
\providecommand{\barridgeapproxerrorCG}[2][\lambda]{\bar A^{\mathrm{CG}}_{{#1},\beta_{\lambda},{#2}}}
\providecommand{\barridgestocherrorCG}[2][\lambda]{\bar S^{\mathrm{CG}}_{{#1},{#2}}}
\providecommand{\barapproxerrorCG}[2][\lambda]{\bar A^{\mathrm{CG}}_{{#1},\gamma,{#2}}}
\providecommand{\barstocherrorCG}[2][\lambda]{\bar S^{\mathrm{CG}}_{{#1},{#2}}}
\providecommand{\ridgecrossCG}[2][\lambda]{C^{\mathrm{CG}}_{{#1},\beta_{\lambda},{#2}}}
\providecommand{\regpredloss}[2][\lambda]{\ell^{\mathrm{in}}_{{#1},\gamma}({#2})}
\providecommand{\regpredrisk}[2][\lambda]{{\cal R}^{\mathrm{in}}_{{#1},\gamma}({#2})}
\providecommand{\outregpredloss}[2][\lambda]{\ell^{\mathrm{out}}_{{#1},\gamma}({#2})}
\providecommand{\outregpredrisk}[2][\lambda]{{\cal R}^{\mathrm{out}}_{{#1},\gamma}({#2})}
\providecommand{\ridgepredloss}[2][\lambda]{\ell^{\mathrm{in}}_{{#1},\beta_{\lambda}}({#2})}
\providecommand{\penpredrisk}[2][\lambda]{{\cal R}^{\mathrm{in}}_{{#1},\beta_0}({#2})}
\providecommand{\emprisk}[2][\lambda]{{\cal E}_{#1}({#2})}
\providecommand{\filterGF}[1][t]{R_{#1}^{\mathrm{GF}}}
\providecommand{\filterCG}[1][t]{R_{#1}^{\mathrm{CG}}}
\providecommand{\filterRR}[1][\lambda,\lambda']{R_{#1}^{\mathrm{RR}}}
\providecommand{\eps}{\varepsilon}

\renewcommand{\le}{\leqslant}
\renewcommand{\ge}{\geqslant}

\DeclareMathOperator{\E}{{\mathbb E}}
\DeclareMathOperator*{\argmin}{argmin}
\DeclareMathOperator{\Var}{Var}
\DeclareMathOperator{\trace}{trace}
\DeclareMathOperator{\diag}{diag}
\DeclareMathOperator{\rank}{rank}

\DeclarePairedDelimiter{\abs}{\lvert}{\rvert}

\DeclarePairedDelimiter{\norm}{\lVert}{\rVert}

\DeclarePairedDelimiterX{\scapro}[2]{\langle}{\rangle}{{#1},{#2}}

\DeclarePairedDelimiter{\ceil}{\lceil}{\rceil}

\begin{document}

\title{Comparing regularisation paths of \\ (conjugate) gradient estimators in ridge regression}

\author{\parbox[t]{13cm}{\centering
Laura Hucker\orcidlink{0009-0003-3293-5281}\textsuperscript{$*$}, 
Markus Rei\ss\textsuperscript{$\dagger$} and 
Thomas Stark\orcidlink{0009-0001-8154-1129}\textsuperscript{$\ddagger$}\\[0.5em]
{\small Institute of Mathematics, Humboldt-Universit{\"a}t zu Berlin, Germany\textsuperscript{$*$,$\dagger$}}\\
{\small Department of Statistics and Operations Research, University of Vienna, Austria\textsuperscript{$\ddagger$}}\\[0.5em]
{\small \mbox{} \textsuperscript{$*$}huckerla@math.hu-berlin.de} \quad
{\small \textsuperscript{$\dagger$}mreiss@math.hu-berlin.de} \quad
{\small \textsuperscript{$\ddagger$}thomas.stark@univie.ac.at}
}}

\date{}

\maketitle

\begin{abstract}
We consider standard gradient descent, gradient flow and conjugate gradients as iterative algorithms for minimising a penalised ridge criterion in linear regression. While it is well known that conjugate gradients exhibit fast numerical convergence, the statistical properties of their iterates are more difficult to assess due to inherent non-linearities and dependencies. On the other hand, standard gradient flow is a linear method with well-known regularising properties when stopped early. By an explicit non-standard error decomposition we are able to bound the prediction error for conjugate gradient iterates by a corresponding prediction error of gradient flow at transformed iteration indices. This way, the risk along the entire regularisation path of conjugate gradient iterations can be compared to that for regularisation paths of standard linear methods like gradient flow and ridge regression. In particular, the oracle conjugate gradient iterate shares the optimality properties of the gradient flow and ridge regression oracles up to a constant factor. Numerical examples show the similarity of the regularisation paths in practice.%

\addvspace{0.5\baselineskip}\noindent%

\addvspace{0.5\baselineskip}\noindent%
\end{abstract}


\section{Introduction}

Conjugate gradient (CG) methods are among the computationally most efficient algorithms for solving large systems of linear equations. In statistics and machine learning they are also known as partial least squares (PLS) methods, see \citet{Hel1990}, \citet[Section~3.5.2]{HasTibFri2009} and the excellent survey \citet{RosKra2006}. The mathematical analysis of CG, however, is challenging due to its greedy, non-linear setup, which leads to intricate dependencies in each iteration. A breakthrough in understanding the implicit regularisation of CG iterates was achieved by \citet{Nem1986} for deterministic inverse problems, which established CG methods in combination with early stopping as one of the leading numerical algorithms for solving and regularising inverse problems, see
\citet{EngHanNeu1996}. Further statistical analysis relied on these results together with high probability concentration bounds to control statistical errors, see \citet{BlaKra2016,SinKriMun2016} and the references therein. Using a different analysis, \citet{FinKri2023} show that CG is superior to principal component regression in latent factor models. Recently in \citet{HucRei2025}, a new error decomposition for CG has paved the way for a non-asymptotic understanding of the approximation and stochastic errors in statistical inverse problems. Inspired by these results, we provide a thorough analysis of the prediction error for all CG iterates, which, after linear interpolation, we call the CG regularisation path, for minimising the penalised least squares criterion of ridge regression.

Consider a potentially high-dimensional linear regression setting under random design, where the number of features $p$ can possibly exceed the sample size $n$. We are given i.i.d.\ observations $(x_i,y_i)$, $i=1,\dots,n$, of $\R^p$-valued feature vectors $x_i$ and $\R$-valued responses $y_i$ satisfying
\begin{equation*}
    y_i = x_i^{\top} \beta_0 + \varepsilon_i, \quad i=1,\dots,n.
\end{equation*}
Here, $\beta_0 \in \R^p$ is the unknown true coefficient vector.
The error variables $\varepsilon_i$ satisfy $\E[\varepsilon_i\,|\,x_i] = 0$ and $\Var(\eps_i\,|\,x_i) = \sigma^2$ for some noise level $\sigma > 0$. Using the notation $y \coloneqq (y_1,\dots,y_n)^\top$, $X \coloneqq (x_1,\dots,x_n)^\top$ and $\varepsilon \coloneqq (\varepsilon_1,\dots,\varepsilon_n)^{\top}$ for the response vector, the feature matrix and the error vector, respectively, we have the linear model
\begin{equation}
    y = X \beta_0 + \varepsilon. \label{EqLinModel}
\end{equation}

The ridge regression (RR) estimator of $\beta_0$ is obtained as minimiser of the penalised least squares criterion
\begin{equation}
    \ridgeest \coloneqq \argmin_{\beta \in \R^p} \emprisk{\beta} \quad \text{with} \quad \emprisk{\beta} \coloneqq \tfrac{1}{2n} \norm{y-X\beta}^2 + \tfrac{\lambda}{2} \norm{\beta}^2, \label{EqRidgeProblem}
\end{equation}
where $\lambda\ge 0$ is the penalty parameter. For $\lambda=0$ we take the minimum-norm solution for $\ridgeest$, see \cref{SecPreliminaries} below. In practice, the RR estimator $\ridgeest$ is calculated by iterative solvers because direct calculations, e.g.\ via an LU decomposition, are numerically too costly when $n$ and $p$ are large.
Standard gradient descent (GD) or conjugate gradients are natural choices of optimisation algorithms in such a scenario, since they are memory-efficient and require only  matrix-vector products in each iteration, allowing for fast and scalable computation in high-dimensional settings. Note that the equally popular quasi-Newton solver BFGS reduces to conjugate gradients, when applied to the quadratic problem \eqref{EqRidgeProblem} \cite{Naz1979}. 
More precisely, as shown by \citet[Section~2]{Naz1979}, when applied to a strictly convex quadratic objective, as in penalised linear regression with $\lambda > 0$, full-memory BFGS initialised with a scaled identity matrix and using an exact line search generates the same iterates as CG. Despite this equivalence, the CG algorithm is generally computationally preferable in practice, as it requires only matrix-vector products and minimal memory, whereas BFGS maintains and updates a dense inverse Hessian approximation in each iteration. This advantage becomes increasingly significant for high-dimensional problems with large $p$. Moreover, when the feature matrix is sparse, the default solver for ridge regression in the standard Python package \textit{scikit-learn} \cite{PedVarGra2011} uses conjugate gradients, profiting from its fast numerical convergence.
Interestingly, the design of more advanced CG algorithms is nowadays driven by statistical applications, ranging from ridge regression to high-dimensional Gaussian sampling, see  \citet{CheHubLin2025} and the references therein.

Gradient descent applied to the ordinary LS problem (the unpenalised criterion in \eqref{EqRidgeProblem} with $\lambda=0$) can be interpreted as an explicit Euler discretisation of the gradient flow (GF) ordinary differential equation
\begin{equation*}
    \dot{\beta}(t) = - \nabla \emprisk{\beta(t)}, \quad \beta(0)=0.
\end{equation*}
Both, discrete-time GD iterates and the regularisation path of its continuous-time analogue GF  exhibit \emph{implicit regularisation}, in the sense that early stopping can substantially improve prediction performance even in the absence of an explicit penalty. The implicit regularisation properties of GF were studied by \citet{AliKolTib2019}.
In particular, the prediction error at each GF iterate is bounded by a small factor times the RR prediction error under a suitably reparametrised ridge penalty $\lambda'$.
In \cref{PropGFRR} below, we generalise this result to gradient flow solving the penalised criterion~\eqref{EqRidgeProblem}. Based on a precise non-asymptotic CG error control, our main result then bounds the prediction error for conjugate gradients after $t$ iterations by the prediction error for gradient flow at iteration $\tau_t$ (\cref{ThmMain}). The time reparametrisation $\tau_t$ is genuinely data-dependent due to the non-linear nature of CG, but explicit in terms of the CG residual polynomial. This comparison result is surprisingly strong and implies that CG has the same (implicit) regularisation effects as gradient flow and a fortiori ridge regression. As a consequence, our result extends the analysis of \citet{AliKolTib2019} to a penalised setting, while also incorporating an additional, computationally more efficient algorithm into the comparison framework. Furthermore, the constant involved in the bound only depends on the spectrum of the empirical covariance matrix of the feature vectors $x_i$, which is discussed for polynomial eigenvalue decay, Marchenko--Pastur type spectral distributions and spiked covariance models in \cref{ExC0}. This result is thus in line with other recent comparison and implicit regularisation theorems like \citet{AliDobTib2020} for  stochastic gradient flow in regression or \citet{WuBarTel2025} for standard gradient descent in logistic regression.

The comparison result in~\cref{ThmMain} allows in particular to bound the CG oracle error (i.e.\ the minimal prediction error along the CG regularisation path) by the corresponding GF and RR oracle errors, see~\cref{CorOracleRisk}. In~\cref{PropGFMon}, we establish that the prediction error along the GF regularisation path is monotonically decreasing for sufficiently large penalties~$\lambda$. As a consequence, a further application of the comparison result in~\cref{ThmMain} yields a corresponding monotonic bound on the CG prediction errors.
Since the main results are derived for the in-sample prediction risk due to the intricate non-linear dependencies in the CG iterates, we provide a transfer from in-sample to out-of-sample prediction risk in \cref{PropOutPredRisk}. This high probability result is essentially valid if the effective rank of the feature covariance matrix is small compared to the sample size. The theoretical results are then illustrated in a high-dimensional simulation and a real data example, where the regularisation paths of CG, GD and RR indeed closely resemble each other, see \cref{SecNum}. A final discussion of our results is given in Section \ref{SecConclusion}, and a table summarising the notation is provided in the \hyperref[TblNotation]{Appendix}.

\section{Ridge regression, gradient flow and conjugate gradients} \label{SecPreliminaries}

Let us first fix some standard notation. The Euclidean norm and scalar product are denoted by $\norm{\cdot}$ and $\scapro{\cdot}{\cdot}$, respectively. For a square matrix $M$ we denote by $\trace(M)$ its trace, by $\norm{M}$ its spectral norm and by $\norm{M}_{HS}=\trace(M^\top M)^{1/2}$ its Hilbert--Schmidt or Frobenius norm. The $p\times p$-identity matrix is denoted by $I_p$. For symmetric matrices $S\in\R^{p\times p}$ with diagonalisation $S=O^\top \diag(\lambda_i)O$ in terms of an orthogonal matrix $O$ and a diagonal matrix of eigenvalues $\lambda_i$, we often apply functional calculus, defining $f(S)=O^\top \diag(f(\lambda_i))O$ for real functions $f$. We write $a\lesssim b$ or $a={\cal O}(b)$ if $a\le Cb$ holds for some constant $C>0$, not depending on parameters involved in the quantities $a,b$. Analogously, $a\thicksim b$ means $a\lesssim b$ and $b\lesssim a$. We set $a\wedge b=\min(a,b)$ and $a\vee b=\max(a,b)$.

For $\lambda>0$ the minimiser of the ridge criterion~\eqref{EqRidgeProblem} is unique by strict convexity, while for $\lambda=0$ we take $\ridgeest[0]=X^+y$, the minimum-norm solution of $y=X\beta$, where $X^{+}$ is the Moore--Penrose pseudoinverse of $X$. Then $\ridgeest\to \ridgeest[0]$ holds as $\lambda\downarrow 0$. Correspondingly, we assume $\beta_0=X^+X\beta_0$, that is, we select the vector $\beta_0$ in model~\eqref{EqLinModel} which  minimises $\norm{\beta}$ among all solutions to $X\beta=X\beta_0$.
In terms of the empirical covariance matrix $\Sigmahat=\frac1n X^\top X$ and its penalised version $\Sigmahatlambda=\Sigmahat+\lambda I_p$, the ridge objective~\eqref{EqRidgeProblem} can be written as
\begin{equation*}
    \emprisk{\beta} = \tfrac12\scapro{\Sigmahatlambda \beta}{\beta}-\tfrac1n\scapro{X^\top y}{\beta}+\tfrac{1}{2n} \norm{y}^2
    =\tfrac12\norm{\Sigmahatlambda^{1/2}\beta-y_\lambda}^2+\tfrac{1}{2n}\norm{y}^2-\tfrac12\norm{y_\lambda}^2
\end{equation*}
with (for $\lambda=0$ we always let $\Sigmahatlambda^{-1}=\widehat\Sigma_0^{-1}=\Sigmahat^+$)
\begin{equation*}
    y_\lambda=\tfrac1n \Sigmahatlambda^{-1/2}X^\top y=\Sigmahatlambda^{1/2}\beta_\lambda+\eps_\lambda,\quad \beta_\lambda=\Sigmahatlambda^{-1}\Sigmahat \beta_0,\quad \eps_\lambda=\tfrac1n \Sigmahatlambda^{-1/2}X^\top \eps.
\end{equation*}
Hence, $\ridgeest$ solves the normal equations
\begin{equation}
    \Sigmahatlambda \beta=\Sigmahatlambda^{1/2}y_\lambda \quad \text{such that} \quad \ridgeest=\beta_\lambda+\Sigmahatlambda^{-1/2}\eps_\lambda.\label{EqbetaRR}
\end{equation}
Inserting the definitions, we find back the classical formula $\ridgeest=\Sigmahatlambda^{-1}\tfrac1n X^\top y$.
We see in \eqref{EqbetaRR} the well-known fact that a penalty $\lambda>0$  introduces a bias $\beta_\lambda-\beta_0$, yet stabilises the stochastic error due to $\Sigmahatlambda> \Sigmahat$ (in the sense that $\Sigmahatlambda-\Sigmahat$ is positive definite).

The standard gradient descent (GD) algorithm with  \emph{step size} (or \emph{learning rate}) $\eta > 0$, initialised at $\ridgegdest{0}=0$, calculates iteratively for $k\in\N$
\begin{equation*}
    \ridgegdest{k} \coloneqq{} \ridgegdest{k-1} - \eta \nabla \emprisk{\ridgegdest{k-1}}
    = \ridgegdest{k-1} - \eta\big(\Sigmahatlambda \ridgegdest{k-1}-\Sigmahatlambda^{1/2}y_\lambda\big). \label{EqDefGD}
\end{equation*}
For a vanishing step size $\eta_k = t/k$ with $k \to \infty$, $t>0$, the  GD estimators $\ridgegdest[\eta_k]{k}$ converge to the \emph{gradient flow (GF) estimator} $\pengradflowest{t}$ satisfying
\begin{equation*}
        \pengradflowest{0}=0,\quad \tfrac{d}{dt} \pengradflowest{t} =  - \nabla \emprisk{\pengradflowest{t}}=-\big(\Sigmahatlambda \pengradflowest{t}-\Sigmahatlambda^{1/2}y_\lambda\big), \quad t\ge 0.
\end{equation*}
Solving this equation explicitly, we obtain the family of gradient flow estimators for $t\ge 0$
\begin{equation}
    \pengradflowest{t} = \Sigmahatlambda^{-1/2}\big(\id_p-\filterGF(\Sigmahatlambda)\big) y_{\lambda} \quad \text{with} \quad \filterGF(x)=\exp(-tx). \label{EqDefPenGF}
\end{equation}

  \begin{algorithm}[t]
    \caption{Penalised CG method}
    \label{AlgCG}
    \begin{algorithmic}[1]
      \vspace{2px}
      \State $\ridgecgest{0} \gets 0$, $q_0 \gets \tfrac{1}{n} X^{\top}y$, $d_0 \gets q_0$, $e_0 \gets X d_0$, $k=1$
      \vspace{2px}
      \While{(no division by zero)}
        \vspace{2px}
        \State $a_k \gets \norm{q_{k-1}}^2 / (\tfrac{1}{n} \norm{e_{k-1}}^2 + \lambda \norm{d_{k-1}}^2)$
        \vspace{2px}
        \State $\ridgecgest{k} \gets \ridgecgest{k-1} + a_k d_{k-1}$
        \vspace{2px}
        \State $q_{k} \gets q_{k-1} - \tfrac{a_k}{n} X^{\top} e_{k-1} - \lambda a_k d_{k-1}$
        \vspace{2px}
        \State $b_k \gets \norm{q_{k}}^2/\norm{q_{k-1}}^2$
        \vspace{2px}
        \State $d_{k} \gets q_{k} + b_k d_{k-1}$
        \vspace{2px}
        \State $e_{k} \gets X d_k$
        \vspace{2px}
        \State $k \gets k+1$
        \vspace{2px}
      \EndWhile
      \vspace{2px}
  \end{algorithmic}
  \end{algorithm}

A usually much faster method to minimise the objective $\emprisk{\beta}$ is conjugate gradients, as described in \cref{AlgCG}, which only requires the input $(y,X,\lambda)$ as well.
The iterations $\ridgecgest{k}$, $k=0,1,2,\ldots$, define the \emph{CG estimators}. Mathematically, they can be expressed in closed form \citep[Minimizing property~II]{PhaHoo2002} as
\begin{equation*}
    \ridgecgest{k} = \Sigmahatlambda^{-1/2} (\id_p - \filterCG[k](\Sigmahatlambda)) y_{\lambda}
\end{equation*}
with \emph{residual polynomial}
\begin{equation}
    \filterCG[k] \coloneqq \argmin_{P_k} \, \norm{P_k(\Sigmahatlambda) y_\lambda}^2, \label{EqDefResPol}
\end{equation}
where the minimiser is taken over all polynomials $P_k$ of degree $k$ with $P_k(0)=1$.
The CG algorithm is initialised at $\ridgecgest{0}=0$, which is natural for the ridge regression problem~\eqref{EqRidgeProblem} with penalty $\frac{\lambda}2\norm{\beta}^2$. For a penalty of the form $\frac{\lambda}2\norm{\beta-\bar\beta}^2$ the iterations would start in $\bar\beta$. Then shifting all parameter values by $\bar\beta$ leads back to our subsequent analysis, compare \citet[Section~7]{EngHanNeu1996}.

Let $s_1\ge\cdots\ge s_p\ge 0$ denote the $p$ eigenvalues (with multiplicities) of $\Sigmahat$ and $v_1,\ldots,v_p\in\R^p$ the corresponding orthonormal basis of eigenvectors. If there are $\tilde p$  distinct eigenvalues of $\Sigmahatlambda$, then there is at least one polynomial $P_{\tilde p}$ of degree $\tilde p$ with value one at zero and vanishing at these eigenvalues so that $\norm{P_{\tilde p}(\Sigmahatlambda) y_\lambda}=0$. Thus, the CG algorithm stops at $k=\tilde p$ at the latest and $\ridgecgest{\tilde p}=\ridgeest$. To avoid indeterminacies before (compare the stopping criterion in \cref{AlgCG}), we assume throughout the paper
\begin{equation*}
    \sum_{j=1,\dots,p: \, s_j = s_i} \scapro{X^\top y}{v_j}^2 > 0, \quad i=1,\dots,p. \label{AssXTy}
\end{equation*}
This holds almost surely, for instance, when the $\eps_i$ have a Lebesgue density.

We extend the CG iterates to the continuous iteration path $[0,\tilde{p}]$ by interpolating linearly between the residual polynomials, as proposed by \citet{HucRei2025}. For $t=k+\alpha$, $k=0,\dots,\tilde{p}-1$, $\alpha \in (0,1]$, the \emph{interpolated CG estimator} is given by
\begin{equation}
    \ridgecgest{t} \coloneqq \Sigmahatlambda^{-1/2} (\id_p - \filterCG(\Sigmahatlambda))y_{\lambda} = (1-\alpha) \ridgecgest{k} + \alpha \ridgecgest{k+1}\label{EqDefInterpolCG}
\end{equation}
with
\begin{equation}
    \filterCG \coloneqq (1-\alpha)\filterCG[k] + \alpha \filterCG[k+1]. \label{EqDefInterpolResPol}
\end{equation}
By \citet[Lemma 4.5]{HucRei2025}, the residual polynomial $\filterCG$ can be rewritten as
\begin{equation*}
\filterCG(x)=\prod_{i=1}^{\ceil{t}} \Big(1-\frac{x}{x_{i,t}}\Big),
\end{equation*}
where $0 < x_{1,t} < \dots < x_{\ceil{t},t} \le s_1+\lambda$ denote the data-dependent zeros of $\filterCG$.

Comparing \eqref{EqDefInterpolCG} and \eqref{EqDefPenGF}, the residual polynomials $\filterCG$ play the same role as the exponentials $\filterGF=\exp(-t \cdot)$ in the gradient flow estimators. The important difference is that $R_t^{\mathrm{CG}}$ depend on $y$ via \eqref{EqDefResPol}. The CG estimators are non-linear in $y$, while gradient flow is  a classical linear regularisation method.  It is remarkable that a rather precise error analysis for conjugate gradients is feasible and allows to compare the errors between the CG and GF regularisation paths.

Finally, by \eqref{EqbetaRR}, we can also express the family of ridge estimators $\hat\beta_{\lambda'}$ for general penalties $\lambda'>0$ in terms of a residual filter applied to $y_\lambda$:
\begin{equation} \label{EqDefRRfilter}
\ridgeest[\lambda'] =\widehat\Sigma_{\lambda'}^{-1}\widehat\Sigma_{\lambda}\ridgeest= \Sigmahatlambda^{-1/2}(I_p-\filterRR(\Sigmahatlambda)) y_{\lambda} \quad \text{with} \quad \filterRR(x) \coloneqq \tfrac{\lambda'-\lambda}{\lambda'-\lambda+x},
\end{equation}
noting $\filterRR(\Sigmahatlambda)=\widehat\Sigma_{\lambda'}^{-1}(\widehat\Sigma_{\lambda'}-\Sigmahatlambda)$.

\section{Error analysis} \label{SecErrorAnalysis}

\subsection{Error of gradient flow and ridge regression} \label{SecErrorMeasures}

To assess the prediction performance of the estimators, we first concentrate on the in-sample prediction errors. A natural criterion is the {\it pure  prediction error}
\begin{equation*}
\tfrac1n\norm{X(\hat\beta-\beta_0)}^2=\norm{\Sigmahat^{1/2}(\hat\beta-\beta_0)}^2
\end{equation*}
for an estimator $\hat\beta$. Inherent to the methodology, however, is the population risk
\begin{equation}\label{EqPopRisk}
 \beta\mapsto \E[\emprisk{\beta}\,|\,X]=\tfrac12\norm{\Sigmahatlambda^{1/2}(\beta-\beta_\lambda)}^2
+\tfrac12\sigma^2+\tfrac12\norm{\Sigmahat^{1/2}\beta_0}^2-\tfrac12\norm{\Sigmahatlambda^{1/2}\beta_\lambda}^2,
\end{equation}
which suggests to measure the error of $\hat\beta$ by the {\it excess penalised prediction error}
\begin{equation*}
\norm{\Sigmahatlambda^{1/2}(\hat\beta-\beta_\lambda)}^2=\tfrac1n\norm{X(\hat\beta-\beta_\lambda)}^2+\lambda\norm{\hat\beta-\beta_\lambda}^2.
\end{equation*}
This is also in line with standard approaches to analyse estimators under penalisation, compare for instance the fundamental inequality for Lasso \citep[Lemma~6.1]{BueGee2011}.

The main handle for analysing conjugate gradient errors are orthogonality properties of the residual polynomial $\filterCG$ with respect to a scalar product induced by the underlying matrix $\Sigmahatlambda$. The inherent CG geometry therefore confines us to consider the family of \emph{regularised in-sample prediction losses}
\begin{equation}\label{EqLin}
    \regpredloss{\hat{\beta}} \coloneqq \norm{\Sigmahatlambda^{1/2}(\hat{\beta}-\gamma)}^2=\tfrac1n\norm{X(\hat\beta-\gamma)}^2
    +\lambda\norm{\hat\beta-\gamma}^2
\end{equation}
for deterministic target vectors $\gamma\in\R^p$. The corresponding {\it prediction risks} are given by
\begin{equation}\label{EqRin}
 \regpredrisk{\hat{\beta}} \coloneqq \E[\regpredloss{\hat{\beta}}\,|\,X].
 \end{equation}
This generalises the excess penalised prediction error for $\gamma=\beta_\lambda$ and gives an upper bound on the pure prediction error for $\gamma=\beta_0$. Further natural target vectors will be $\gamma=\beta_{\lambda'}$, the conditional expectation of the ridge estimator $\ridgeest[\lambda']$ for general penalties $\lambda'>0$, which generalise the first two cases.

\begin{proposition}[Standard error decomposition]\label{PropErrDecomp}
Consider for any residual filter function $R:[0,\infty)\to\R$ the estimator
\begin{equation*}
\hat\beta \coloneqq \Sigmahatlambda^{-1/2}(I_p-R(\Sigmahatlambda))y_\lambda.
\end{equation*}
Then the penalised prediction loss satisfies
    \begin{equation*}
        \regpredloss{\hat\beta} = A_{\lambda,\gamma}(R) + S_{\lambda}(R) - 2C_{\lambda,\gamma}(R)\label{EqRidgePredLossGFDecomp}
    \end{equation*}
with approximation, stochastic and cross term errors
\begin{align*}
A_{\lambda,\gamma}(R)&\coloneqq \norm{\Sigmahatlambda^{1/2}(R(\Sigmahatlambda)\beta_\lambda+(\gamma-\beta_\lambda))}^2,\\ S_{\lambda}(R)&\coloneqq \norm{(I_p-R(\Sigmahatlambda))\eps_\lambda}^2,\\
C_{\lambda,\gamma}(R)&\coloneqq \scapro{\Sigmahatlambda^{1/2}(R(\Sigmahatlambda)\beta_\lambda+(\gamma-\beta_\lambda))}
{(I_p-R(\Sigmahatlambda))\eps_\lambda}.
\end{align*}
If $R$ is deterministic, then the corresponding risk of $\hat\beta=\Sigmahatlambda^{-1/2}(I_p-R(\Sigmahatlambda))y_\lambda$ is
\begin{equation*}
\regpredrisk{\hat\beta}=A_{\lambda,\gamma}(R) + \tfrac{\sigma^2}{n}\trace\big((I_p-R(\Sigmahatlambda))^2\Sigmahatlambda^{-1}\Sigmahat\big).
\end{equation*}
\end{proposition}

\begin{proof}
This follows directly from the definitions and the fact that the cross term has conditional mean zero when $R$ is deterministic. For the conditional variance term note $\E[\eps_\lambda\eps_\lambda^\top\,|\,X]=\frac{\sigma^2}{n}\Sigmahatlambda^{-1}\Sigmahat$.
\end{proof}

We can immediately compare the regularisation path of gradient flow in time $t> 0$ with that of ridge regression with penalties larger than $\lambda$. In particular, solving the ridge problem~\eqref{EqRidgeProblem} for some prespecified minimal $\lambda\ge 0$ via gradient flow yields an iteration path of regularised estimators with the same risk as ridge estimators with larger penalties, up to a small numerical factor. This generalises the result by \citet{AliKolTib2019} derived for $\lambda=0$, $\gamma=\beta_0$. To obtain the more general result, we need to require a geometric condition on the target vector $\gamma$ in the loss.

\begin{proposition}\label{PropGFRR}
If the target vector $\gamma$ satisfies
\begin{equation}\label{EqgammaCond1}
\scapro{\Sigmahatlambda ((I_p+t\Sigmahatlambda)^{-1}-\exp(-t\Sigmahatlambda))\beta_\lambda }{ \gamma-\beta_\lambda}\ge 0,
\end{equation}
then the prediction risk of gradient flow at time $t>0$ is bounded by that of ridge regression with penalty $\lambda+1/t$ via
\begin{equation*}
\regpredrisk{\pengradflowest{t}}\le 1.2985^2 \regpredrisk{\ridgeest[\lambda+1/t]}.
\end{equation*}
Condition~\eqref{EqgammaCond1} and the result hold in particular for all $\gamma=\beta_{\lambda'}$ with $\lambda'\in[0,\lambda]$.
\end{proposition}

\begin{proof}
By \citet[Lemma~7]{AliKolTib2019}, we have for $x,t\ge 0$
\begin{equation*}
\exp(-tx)\le (1+tx)^{-1},\quad 1-\exp(-tx)\le 1.2985(1-(1+tx)^{-1}).
\end{equation*}
This yields directly
\begin{align*}
S_{\lambda}(\filterGF)&=\norm{(I_p-\exp(-t\Sigmahatlambda))\eps_\lambda}^2 \\
&\le 1.2985^2\norm{(I_p-(I_p+t\Sigmahatlambda)^{-1})\eps_\lambda}^2= 1.2985^2 S_{\lambda}(\filterRR[\lambda,\lambda+1/t]).
\end{align*}
For the approximation error, we calculate
\begin{align*}
&A_{\lambda,\gamma}(\filterGF)-A_{\lambda,\gamma}(\filterRR[\lambda,\lambda+1/t])\\
&= \norm{\Sigmahatlambda^{1/2}\exp(-t\Sigmahatlambda)\beta_\lambda}^2 -\norm{\Sigmahatlambda^{1/2}(I_p+t\Sigmahatlambda)^{-1}\beta_\lambda}^2 \\
&\hphantom{={}} +2\scapro{\Sigmahatlambda(\gamma-\beta_\lambda) }{(\exp(-t\Sigmahatlambda)-(I_p+t\Sigmahatlambda)^{-1})\beta_\lambda} \\
&\le -2\scapro{\Sigmahatlambda ((I_p+t\Sigmahatlambda)^{-1}-\exp(-t\Sigmahatlambda))\beta_\lambda }{ \gamma-\beta_\lambda}\le 0
\end{align*}
by Condition~\eqref{EqgammaCond1}. This gives the bound.

For $\gamma=\beta_{\lambda'}=\widehat\Sigma_{\lambda'}^{-1}\Sigmahatlambda \beta_\lambda$ we infer $\gamma-\beta_\lambda=(\lambda-\lambda')\widehat\Sigma_{\lambda'}^{-1} \beta_\lambda$ and Condition~\eqref{EqgammaCond1} becomes
\begin{equation*}
(\lambda-\lambda')\scapro{\Sigmahatlambda ((I_p+t\Sigmahatlambda)^{-1}-\exp(-t\Sigmahatlambda))\beta_\lambda }{\widehat\Sigma_{\lambda'}^{-1} \beta_\lambda}\ge 0.
\end{equation*}
Since $\widehat\Sigma_{\lambda'}^{-1}$ and $\Sigmahatlambda ((I_p+t\Sigmahatlambda)^{-1}-\exp(-t\Sigmahatlambda))$ are positive semi-definite for  $t\ge 0$ and are jointly diagonalised by the eigenvectors of $\Sigmahat$, Condition~\eqref{EqgammaCond1} is satisfied for $\lambda-\lambda'\ge 0$.
\end{proof}

\subsection{Error of conjugate gradients}

For conjugate gradients, the data-dependent residual polynomial $\filterCG$ cannot be controlled globally and we shall use \cref{PropErrDecomp} only for the RR and GF estimators. For CG we pursue a different error decomposition.

Recall that we denote by $x_{1,t}$ the smallest zero on $[0,\infty)$ of the $t$-th residual polynomial $\filterCG$, $t \in (0,\tilde{p}]$. For $x\ge 0$ introduce the functions
\begin{equation}
\filterCG[t,<](x) \coloneqq \filterCG(x){\bf 1}(x<x_{1,t}), \quad
\filterCG[t,>](x) \coloneqq \filterCG(x){\bf 1}(x>x_{1,t}) \label{EqDefTruncatedResPols}
\end{equation}
and set $\filterCG[0,<] \coloneqq \filterCG[0] \equiv 1$, $\filterCG[0,>] \equiv 0$.
With this we can transfer the explicit error decomposition for conjugate gradients from \citet[Proposition~5.2]{HucRei2025}.

\begin{proposition}\label{PropErrorDecomps}
        The penalised CG estimator at $t \in [0,\tilde{p}]$ satisfies for $\gamma=\beta_\lambda$
    \begin{equation}
        \ridgepredloss{\ridgecgest{t}} = \ridgeapproxerrorCG{t} + \ridgestocherrorCG{t} - 2\ridgecrossCG{t} \le 2\ridgeapproxerrorCG{t} + 2\ridgestocherrorCG{t} \label{EqCGDecomp}
    \end{equation}
    with approximation, stochastic  and cross term errors given by
    \begin{align}
        \ridgeapproxerrorCG{t} &\coloneqq \norm{\Sigmahatlambda^{1/2}\filterCG[t,<](\Sigmahatlambda)^{1/2}\beta_{\lambda}}^2 + \norm{\filterCG(\Sigmahatlambda)y_{\lambda}}^2 - \norm{\filterCG[t,<](\Sigmahatlambda)^{1/2}y_{\lambda}}^2 \notag
        \\
        &\le \norm{\Sigmahatlambda^{1/2}\filterCG[t,<](\Sigmahatlambda)^{1/2}\beta_{\lambda}}^2, \label{EqRidgePredLossApproxErrorCGBounds} \\
        \ridgestocherrorCG{t} &\coloneqq \norm{(I_p-\filterCG[t,<](\Sigmahatlambda))^{1/2}\varepsilon_{\lambda}}^2, \label{EqRidgePredLossStochErrorCGBounds} \quad
        \ridgecrossCG{t} \coloneqq \scapro{\filterCG[t,>](\Sigmahatlambda)y_{\lambda}}{\varepsilon_{\lambda}}.
    \end{align}
\end{proposition}

\begin{figure}
    \centering
    \includegraphics[width=\textwidth]{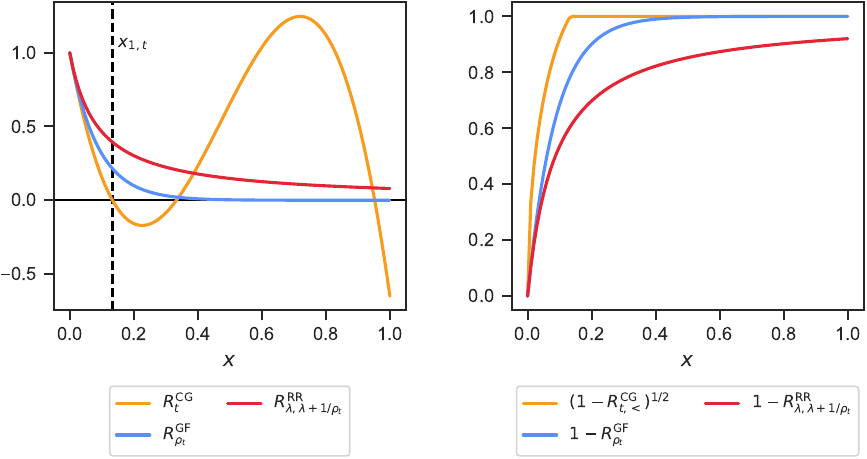}
    \caption{Left: Filters $\filterCG$ (orange), $\filterGF[\rho_t]$ (blue), $\filterRR[\lambda,\lambda+1/\rho_t]$ (red) and the first zero $x_{1,t}$ of $\filterCG$ (dashed vertical). Right: Filters $(1-\filterCG[t,<])^{1/2}$ (orange), $1-\filterGF[\rho_t]$ (blue), $1-\filterRR[\lambda,\lambda+1/\rho_t]$ (red) for the stochastic errors.}
    \label{FigResidualPolynomial}
\end{figure}

The two bounds appearing in \cref{PropErrorDecomps} provide key insights for the further analysis. First, the cross term $\ridgecrossCG{t}$ does not have mean zero and cannot be controlled directly due to the randomness of all quantities involved. We shall therefore work with the bound in \eqref{EqCGDecomp}, at the cost of a factor $2$. Second, {\it Nemirovskii's trick}~\cite[(3.14)]{Nem1986} yields the inequality in \eqref{EqRidgePredLossApproxErrorCGBounds} and thus reduces the control of the approximation error to the behaviour of the residual polynomial $\filterCG$ on $[0,x_{1,t}]$ at the cost of replacing the residual filters $\filterCG,1-\filterCG$ by $(\filterCG[t,<])^{1/2},(1-\filterCG[t,<])^{1/2}$ in the approximation and stochastic error terms, respectively.

\cref{FigResidualPolynomial} illustrates the difficulty in analysing the CG residual polynomial globally. The left symbolic plot shows a CG residual polynomial $\filterCG$  and the corresponding residual filters $\filterGF[\rho_t]$, $\filterRR[\lambda,\lambda+1/\rho_t]$  for gradient flow and ridge regression, respectively, with $\rho_t \coloneqq \abs{(\filterCG)'(0)}$. $\filterCG(x)$ is up to its first zero $x_{1,t}$ upper bounded by  $\filterGF[\rho_t](x)$, but afterwards it may fluctuate strongly, taking also values outside of $[0,1]$. Indeed, \citet[Lemma~4.7]{HucRei2025} state that the residual polynomial $\filterCG$ is tightly bounded on $[0,x_{1,t}]$ in terms of its derivative at zero:
\begin{equation}
(1-\rho_t x)_+ \le \filterCG(x) \le \exp(-\rho_t x), \quad x \in [0,x_{1,t}]. \label{EqRtbounds}
\end{equation}
The right plot displays the filter functions that appear in the stochastic error terms. For CG $(1-\filterCG[t,<](x))^{1/2}$ grows like $(\rho_t x)^{1/2}$ for $x$ near zero, while $1-\filterGF[\rho_t]$ and $1-\filterRR[\lambda,\lambda+1/\rho_t]$ grow linearly like $\rho_t x$. \cref{PropErrorDecomps} and the bounds in \eqref{EqRtbounds} yield already a more explicit universal bound.

\begin{corollary}\label{CorCGloss1}
        The penalised CG estimator at $t \in [0,\tilde{p}]$ satisfies for $\gamma=\beta_\lambda$
    \begin{equation*}
        \ridgepredloss{\ridgecgest{t}}  \le 2\barridgeapproxerrorCG{t} + 2\barridgestocherrorCG{t}
    \end{equation*}
    with approximation and stochastic  error bounds given by
    \begin{equation*}
        \barridgeapproxerrorCG{t} \coloneqq \norm{\Sigmahatlambda^{1/2}{\exp(-\rho_t\Sigmahatlambda /2)}\beta_{\lambda}}^2 , \quad
        \barridgestocherrorCG{t} \coloneqq \norm{(\rho_t \Sigmahatlambda\wedge 1)^{1/2}\varepsilon_{\lambda}}^2 \quad \text{for} \quad \rho_t=\abs{(\filterCG)'(0)},
    \end{equation*}
    where $(\rho_t \Sigmahatlambda\wedge 1)^{1/2}$ denotes the matrix $f(\Sigmahatlambda)$ for $f(x)=(\rho_t x\wedge 1)^{1/2}$, $x\ge 0$.
\end{corollary}

We proceed to the case of general target vectors $\gamma$ and identify a geometric condition under which the results for the intrinsic parameter $\gamma=\beta_\lambda$ can be extended.

\begin{proposition}\label{PropCGlossgeneral}
Let $t \in [0,\tilde{p}]$. Assume that the target vector $\gamma\in\R^p$ satisfies  for $\rho_t=\abs{(\filterCG)'(0)}$
   \begin{equation}\label{gammacond}
     \scapro{\Sigmahatlambda\exp(-\rho_t\Sigmahatlambda /2)\beta_\lambda}{\gamma-\beta_\lambda}\ge 0.
    \end{equation}
Then the regularised in-sample prediction loss for the CG estimator is bounded by
    \begin{equation*}
        \regpredloss{\ridgecgest{t}}  \le 4\barapproxerrorCG{t} + 4\barstocherrorCG{t} \label{EqRidgePredLossCGDecomp}
    \end{equation*}
    with $\gamma$-dependent approximation error bound
    \begin{equation*}
        \barapproxerrorCG{t} \coloneqq \norm{\Sigmahatlambda^{1/2}(\beta_\lambda-\exp(-\rho_t\Sigmahatlambda /2)\beta_{\lambda}-\gamma)}^2. \label{EqRidgePredLossStochErrorCGBounds2}
    \end{equation*}
Condition~\eqref{gammacond} is always satisfied for $\gamma=\beta_{\lambda'}$, $\lambda'\in[0,\lambda]$.
\end{proposition}

\begin{proof}
Using the general inequality $(A+B)^2\le 2(A^2+B^2)$, we obtain from \cref{CorCGloss1}
\begin{align*}
\regpredloss{\ridgecgest{t}} &= \norm{\Sigmahatlambda^{1/2}(\ridgecgest{t}-\beta_\lambda-(\gamma-\beta_\lambda))}^2\\
&\le 2\Big(2\norm{\Sigmahatlambda^{1/2}{\exp(-\rho_t\Sigmahatlambda /2)}\beta_{\lambda}}^2+ 2\norm{(\rho_t \Sigmahatlambda\wedge 1)^{1/2}\varepsilon_{\lambda}}^2+ \norm{\Sigmahatlambda^{1/2}(\gamma-\beta_\lambda)}^2\Big)\\
&\le 4\norm{\Sigmahatlambda^{1/2}(\exp(-\rho_t\Sigmahatlambda /2)\beta_\lambda+\gamma-\beta_\lambda)}^2+ 4\norm{(\rho_t \Sigmahatlambda\wedge 1)^{1/2}\varepsilon_{\lambda}}^2\\
&=4\barapproxerrorCG{t} + 4\barstocherrorCG{t},
\end{align*}
where Condition~\eqref{gammacond} was used in the last line to ensure that the approximation error is larger or equal to the sum of the squared norms of its parts.

For $\gamma=\beta_{\lambda'}=\widehat\Sigma_{\lambda'}^{-1}\Sigmahatlambda \beta_\lambda$, Condition~\eqref{gammacond} requires
\begin{equation*}
\scapro{\Sigmahatlambda\exp(-\rho_t\Sigmahatlambda /2)\beta_\lambda}{(\widehat\Sigma_{\lambda'}^{-1}\Sigmahatlambda-I_p)\beta_\lambda}
=(\lambda-\lambda')\scapro{\Sigmahatlambda\exp(-\rho_t\Sigmahatlambda /2)\beta_\lambda}{\widehat\Sigma_{\lambda'}^{-1}\beta_\lambda}
\end{equation*}
to be larger or equal to zero. Since $\Sigmahatlambda$ and $\widehat\Sigma_{\lambda'}$ are positive semi-definite and jointly diagonalisable, this is the case for any choice of $\lambda'\in[0,\lambda]$.
\end{proof}

The table in the \hyperref[TblNotation]{Appendix} summarises the key quantities introduced so far.

\subsection{Main results}

The main point in \cref{PropCGlossgeneral} is that the CG approximation error bound $\barapproxerrorCG{t}$ at time $t$ equals exactly the GF approximation error $A_{\lambda,\gamma}(\filterGF[\rho_t/2])$ at the random time $\rho_t/2$, where $\rho_t=\abs{(\filterCG)'(0)}$ is in principle computable from the data. Our aim is thus to compare the CG and GF errors up to a random time shift. Inverting the transformation $t\mapsto \rho_t/2$, we introduce  the random times
    \begin{equation*}
        \tau_t \coloneqq \inf \big\{ \tilde{t} \in [0,\tilde{p}] \, \big| \, \abs{(\filterCG[\tilde{t}])'(0)} \ge 2t \big\} \wedge \tilde{p},\quad t\ge 0.
    \end{equation*}
Note that the mapping $t \mapsto \abs{(\filterCG)'(0)}$ is strictly increasing, continuous and due to our CG interpolation scheme piecewise linear, which transfers directly to $t\mapsto\tau_t$ by inversion.

We obtain a first general comparison result for the  risk of CG. Note that generalisations to correlated errors $\eps_i$ are feasible, but the expected stochastic errors become much less explicit.

\begin{corollary}\label{CorCGGF1}
Under Condition~\eqref{gammacond} the prediction risk of CG at time $\tau_t$ satisfies
\begin{equation*}
\regpredrisk{\ridgecgest{\tau_t}}  \le 4A_{\lambda,\gamma}(\filterGF) + \tfrac{4\sigma^2}{n}\trace\big((2t\wedge \Sigmahatlambda^{-1})\Sigmahat\big).
\end{equation*}
\end{corollary}

\begin{proof}
By definition of $\filterCG[\tilde{t}]$ via linear interpolation, $\rho_{\tilde t}=\abs{(\filterCG[\tilde{t}])'(0)}$ is continuous in $\tilde t$ and by \citep[Lemma~4.5]{HucRei2025} it grows from zero at $\tilde t=0$ to $\sum_{i=1}^{\tilde p}(\tilde s_i+\lambda)^{-1}$ at $\tilde t=\tilde p$, where $\tilde s_1,\ldots,\tilde s_{\tilde p}$ denote the $\tilde p$ distinct eigenvalues of $\Sigmahat$. We conclude that $\rho_{\tau_t}=2t$ holds for $t\in[0,\frac12\sum_{i=1}^{\tilde p}(\tilde s_i+\lambda)^{-1}]$. In this case the claim follows directly from \cref{PropCGlossgeneral}, noting again $\E[\eps_\lambda\eps_\lambda^\top\,|\,X]=\tfrac{\sigma^2}{n}\Sigmahatlambda^{-1}\Sigmahat$.

For $t>\frac12\sum_{i=1}^{\tilde p}(\tilde s_i+\lambda)^{-1}$ we have $\tau_t=\tilde p$ and we obtain the terminal CG iterate $\ridgecgest{\tilde p}=\ridgeest$ so that
\begin{equation*}
\regpredrisk{\ridgecgest{\tau_t}}=\E[\norm{\eps_\lambda}^2\,|\,X]
=\tfrac{\sigma^2}{n}\trace\big(\Sigmahatlambda^{-1}\Sigmahat\big).
\end{equation*}
On the other hand, we have $2t>\max_i(\tilde s_i+\lambda)^{-1}=\norm{\Sigmahatlambda^{-1}}$ for this case, implying $(2t\wedge \Sigmahatlambda^{-1})=\Sigmahatlambda^{-1}$. This shows $\regpredrisk{\ridgecgest{\tau_t}}= \tfrac{\sigma^2}{n}\trace\big((2t\wedge \Sigmahatlambda^{-1})\Sigmahat\big)$ and yields the asserted bound also in this case.
\end{proof}

Based on the previous result, it remains to compare the CG and GF stochastic errors. We remark that $(1-\exp(-tx))^2\le 1-\exp(-tx)\le 2tx\wedge 1$ for $x,t\ge 0$ always implies for the stochastic error bound of \cref{CorCGGF1}
\begin{equation*}
\tfrac{4\sigma^2}{n}\trace\big((2t\wedge \Sigmahatlambda^{-1})\Sigmahat\big)\ge \tfrac{\sigma^2}{n}\trace\big((I_p-\exp(-t\Sigmahatlambda))^2\Sigmahatlambda^{-1}\Sigmahat\big)=S_\lambda(\filterGF).
\end{equation*}
The left-hand side may in fact be much larger than $S_\lambda(\filterGF)$, just note that for $t\downarrow 0$ it is of order $t$, while the right-hand side is of order $t^2$. The ultimate reason is that the CG stochastic error term is $\ridgestocherrorCG{t} = \norm{(I_p-\filterCG[t,<](\Sigmahatlambda))^{1/2}\varepsilon_{\lambda}}^2$ and not $\norm{(I_p-\filterCG[t,<](\Sigmahatlambda))\varepsilon_{\lambda}}^2$, a term formed in analogy to the standard error decomposition in \cref{PropErrDecomp}. Still, it only depends on the spectrum of the empirical covariance matrix $\Sigmahat$ and we can find easy spectral conditions for a valid comparison between CG and GF error. We arrive at our main comparison theorem between CG and GF.

\begin{theorem}\label{ThmMain}
Assume $s_1+\lambda > 0$, where $s_1\ge\cdots\ge s_p\ge 0$ denote the eigenvalues of $\Sigmahat$. For $t\ge \frac1{2(s_1+\lambda)}$ set
\begin{equation*}
C_{t,\lambda} \coloneqq \begin{cases}
\textstyle{\big(\sum_{j\ge i_t} s_j\big)\big(\sum_{j<i_t}\tfrac{(s_{i_t}+\lambda)s_j}{s_j+\lambda}+\sum_{j\ge i_t}\tfrac{(s_j+\lambda)s_j}{s_{i_t-1}+\lambda}\big)^{-1}},& t< \tfrac12(s_p+\lambda)^{-1},\\
0,& t\ge \tfrac12(s_p+\lambda)^{-1},
\end{cases}
\end{equation*}
with $i_t\coloneqq \min\{j\,|\,s_j<(2t)^{-1}-\lambda\}\in \{2,\ldots,p\}$ for $t\in [\frac12(s_1+\lambda)^{-1},\frac12(s_p+\lambda)^{-1})$.
Then $C_{t,\lambda}\le C_{t,0}$ holds
 and under Condition~\eqref{gammacond} for $\gamma$ we have
\begin{equation}\label{EqMainbound}
 \regpredrisk{\ridgecgest{\tau_t}}  \le \tfrac{4(1+C_{t,\lambda})}{(1-e^{-1/2})^2}\regpredrisk{\pengradflowest{t}}, \quad t\ge \tfrac1{2\norm{\Sigmahatlambda}}.
\end{equation}
\end{theorem}

\begin{remark}
Concerning the condition $t\ge \tfrac1{2(s_1+\lambda)}$, note that for $t\in[0,\frac1{2(s_1+\lambda)}]$ we have $\filterGF(\Sigmahatlambda)\ge e^{-1/2}I_p$, and the residual filter of gradient flow does not provide any significant regularisation.

Recall that for $\gamma=\beta_0$ the pure prediction loss is bounded by the penalised prediction loss so that in this case the left-hand side in \eqref{EqMainbound} can be replaced by $\tfrac{1}{n} \E[\norm{X(\ridgecgest{\tau_t}-\beta_0)}^2 \,|\, X]$. For $\gamma=\beta_\lambda$ we can add the difference between population and excess risk in \eqref{EqPopRisk}, which does not involve the estimator, to both sides of \eqref{EqMainbound}. This gives
\begin{equation*} \E[\emprisk{\ridgecgest{\tau_t}}\,|\,X]  \le \tfrac{4(1+C_{t,\lambda})}{(1-e^{-1/2})^2}\E[\emprisk{\pengradflowest{t}}\,|\,X].
\end{equation*}
\end{remark}

\begin{proof}[Proof of \cref{ThmMain}]
We consider the stochastic risk for $\ridgecgest{\tau_t}$ from \cref{CorCGGF1} and bound
\begin{align*}
\trace\big((2t\wedge \Sigmahatlambda^{-1})\Sigmahat\big) &\le \trace\big((2t\wedge \Sigmahatlambda^{-1})^2\Sigmahatlambda\Sigmahat\big)+ \trace\big(2t\Sigmahat {\bf 1}(2t<\Sigmahatlambda^{-1})\big).
\end{align*}
On the other hand, the conditional variance term for $\pengradflowest{t}$ from \cref{PropErrDecomp} can be bounded from below using $1-e^{-x}\ge (1-e^{-1/2})(2x\wedge 1)$, $x\ge 0$,
\begin{equation*}
    \trace\big((I_p-\filterGF(\Sigmahatlambda))^2\Sigmahatlambda^{-1}\Sigmahat\big) \ge (1-e^{-1/2})^2
\trace\big((2t\wedge \Sigmahatlambda^{-1})^2\Sigmahatlambda\Sigmahat\big).
\end{equation*}
We thus obtain by \cref{CorCGGF1}
\begin{align*}
\regpredrisk{\ridgecgest{\tau_t}}  &\le 4(1-e^{-1/2})^{-2} \Big(1+\tfrac{\trace(2t\Sigmahat {\bf 1}(2t<\Sigmahatlambda^{-1}))}{ \trace((2t\wedge \Sigmahatlambda^{-1})^2\Sigmahatlambda\Sigmahat)}\Big) \regpredrisk{\pengradflowest{t}}.
\end{align*}
Since by definition $s_{i_t}+\lambda<(2t)^{-1}\le s_{i_t-1}+\lambda$ holds for $t\in [\frac12(s_1+\lambda)^{-1},\frac12(s_p+\lambda)^{-1})$, we obtain
\begin{align*}
\tfrac{\trace(2t\Sigmahat {\bf 1}(2t<\Sigmahatlambda^{-1}))}{ \trace((2t\wedge \Sigmahatlambda^{-1})^2\Sigmahatlambda\Sigmahat)} &= \tfrac{\sum_{j\ge i_t} s_j}{\sum_{j<i_t}(2t)^{-1}(s_j+\lambda)^{-1}s_j+\sum_{j\ge i_t}2t(s_j+\lambda)s_j}\\
&\le \tfrac{\sum_{j\ge i_t} s_j}{\sum_{j<i_t}(s_{i_t}+\lambda)(s_j+\lambda)^{-1}s_j+\sum_{j\ge i_t}(s_{i_t-1}+\lambda)^{-1}(s_j+\lambda)s_j}=C_{t,\lambda}.
\end{align*}
For $t\ge \frac12(s_p+\lambda)^{-1}$ the trace in the numerator vanishes and we may use $C_{t,\lambda}=0$.
This gives the comparison of risk bound \eqref{EqMainbound}.
Applying the general inequality $\frac{A+\lambda}{B+\lambda}\ge \frac AB$ for $B\ge A> 0$ to the terms in the denominator of $C_{t,\lambda}$ yields
$C_{t,\lambda}\le C_{t,0}$.
\end{proof}

Let us consider typical examples of feature matrices and check whether the constant $C_{t,\lambda}$ can be bounded uniformly in $t$. In the sequel, we set
\begin{equation} \label{EqDefCbar}
    \overline{C}_{\lambda} \coloneqq \sup_{t \ge (2\norm{\Sigmahatlambda})^{-1}} C_{t,\lambda}.
\end{equation}

\begin{example}\label{ExC0}\
\begin{enumerate}[(i)]
\item
In view of $C_{t,\lambda}\le C_{t,0}$ let us first study $\lambda=0$.
For $t\in [\frac12s_1^{-1},\frac12s_p^{-1})$ we have
\begin{equation}\label{EqCt0} C_{t,0}=\tfrac{\sum_{j\ge i_t} (s_j/s_{i_t})}{(i_t-1)+(s_{i_t}/s_{i_t-1})\sum_{j\ge i_t}(s_j/s_{i_t})^2}.
\end{equation}
Suppose that for some $C>0$, $\alpha>1$,
\begin{equation}\label{CondC1}
\forall\,2\le i\le j\le p:\;s_j/s_i\le C(j/i)^{-\alpha}
\end{equation}
holds. Then
by Riemann sum approximation and $\alpha/(i_t-1) \le \alpha$ for $i_t\ge 2$
\begin{equation*}
    C_{t,0} \le \tfrac{C\sum_{j\ge i_t} (j/i_t)^{-\alpha}}{i_t-1}\le C\tfrac{1+i_t\int_{1}^\infty x^{-\alpha}dx}{i_t-1}\le \tfrac{C(\alpha+1)}{\alpha-1}
\end{equation*}
follows and all $C_{t,\lambda}$ are uniformly bounded by $\overline{C}_{\lambda} \le \tfrac{C(\alpha+1)}{\alpha-1}$. The Condition~\eqref{CondC1} comprises polynomial decay $s_i\thicksim i^{-\alpha}$, $\alpha>1$, as well as geometric decay $s_i\thicksim q^i$, $q\in (0,1)$, of the eigenvalues. Yet, for slowly decaying $s_i\thicksim i^{-\alpha}$ we have $C_{t,0}\thicksim (p/i_t)^{1-\alpha}$ for $\alpha\in(1/2,1)$ and $C_{t,0}\thicksim (p/i_t)^{\alpha}$ for $\alpha\in(0,1/2)$, which becomes unbounded for dimension $p\to\infty$. The  different growth of the CG and GF stochastic error terms in $p$ can also be checked directly in this case. It is not due to a possibly suboptimal bound in \cref{ThmMain}.

\item Consider the slow eigenvalue decay
\begin{equation}\label{CondC2}
\forall\,i=2,\ldots,p:\; \tfrac{s_i}{s_{i-1}}\sum_{j=i}^p\tfrac{s_j^2}{s_i^2}\ge c(p-i+1)
\end{equation}
for some $c>0$. Then, bounding the numerator in \eqref{EqCt0} by $c(p-i_t+1)$, we obtain the uniform bound $\overline{C}_{\lambda} \le 1/c$. As a typical example consider $s_i\thicksim (1-i/(p+1))^\beta$ for some $\beta>0$. Then
\begin{equation*}
    \sum_{j=i}^ps_j^2\thicksim \sum_{j=i}^p(1-j/(p+1))^{2\beta}=(p+1)^{-2\beta}\sum_{k=1}^{p-i+1}k^{2\beta}\thicksim (p-i+1)s_i^2
\end{equation*}
and $s_i/s_{i-1}\thicksim 1$ follow so that Condition \eqref{CondC2} is satisfied for some $c>0$. Given a density $f:[0,1]\to\R^+$ with invertible distribution function $F(x)=\int_0^xf(y)dy$, we might think of a regular sample with $s_i=F^{-1}(1-i/(p+1))$. Hence, if the quantile function satisfies $F^{-1}(x)\thicksim x^\beta$ near zero, e.g.\ if $f(x)\thicksim x^{1/\beta-1}$, then $C_{t,0}$ is uniformly bounded in $t$. The same conclusion holds for any shifted support interval of $f$ instead of $[0,1]$, in particular for the Marchenko--Pastur density ($\beta=2$ or $\beta=2/3$).

\item Finally, we consider a spiked covariance model for $\Sigmahat$ with eigenvalues $s_i\ge 1$ for $i=1,\ldots, r$ and $s_i\le\eps$ with small $\eps\in(0,1)$ for $i=r+1,\ldots,p$. Due to the large gap $s_r-s_{r+1}\ge 1-\eps$, the situation is more intricate and we consider directly $C_{t,\lambda}$. We have $i_t\le r$ for $t<\frac12(s_r+\lambda)^{-1}\le \frac{1}{2} (1+\lambda)^{-1}$ and $i_t\ge r+2$ for $t\ge \frac12(s_{r+1}+\lambda)^{-1}\ge \frac12(\eps+\lambda)^{-1}$. We obtain
\begin{equation*}
    C_{t,\lambda} \thicksim\begin{cases} \frac{(r-i_t)_++(p-r)\eps}{r+(p-r)(1+\lambda)^{-1}(\eps+\lambda)\eps}\lesssim \frac{(p-r)\eps}{r}\wedge \frac{1+\lambda}{\eps+\lambda},& i_t\le r+1,\\ \frac{(p-i_t+1)\eps}{r+(p-r)\eps}\lesssim 1 ,& i_t\ge r+2.
    \end{cases}
\end{equation*}
We see that $C_{t,\lambda}$ is always uniformly bounded for $t\ge \frac12(\eps+\lambda)^{-1}$, when the bulk of the $p-r$ small eigenvalues starts to induce significant regularisation. For smaller $t$, however, conjugate gradients reduces the stochastic error within the bulk by an even smaller factor than gradient flow and this might nevertheless have a global effect when $p-r$ is large. Still, $C_{t,\lambda}$ remains uniformly bounded for all $t\le\frac12(s_1+\lambda)^{-1}$ when $\eps\lesssim\frac{r}{p-r}$ or $\lambda\gtrsim 1$. The first condition can be relaxed to $\sum_{i\ge r+1}s_i\lesssim r$ on the trace. Especially, in the high-dimensional regime $p>n$ the condition $\eps\lesssim \frac{r}{n-r}$ suffices because of $\rank(\Sigmahat)\le n$.
\end{enumerate}
\end{example}

Our main theorem has a direct consequence for the best risk of CG along the CG iteration path. In particular, this allows to transfer oracle risk bounds for gradient flow or ridge regression to conjugate gradients. Note that this is even new for the unpenalised problem with $\lambda=0$.

\begin{corollary}\label{CorOracleRisk}
Let $\gamma=\beta_{\lambda'}$ for $\lambda'\in[0,\lambda]$. Then with $C_{t,\lambda}$ from \cref{ThmMain} and $\overline{C}_{\lambda}$ from \eqref{EqDefCbar} we have
\begin{equation*}
\inf_{t\ge 0}\regpredrisk{\ridgecgest{t}} \le 25.9 (1+\overline{C}_{\lambda}) \! \inf_{t\ge  (2\norm{\Sigmahatlambda})^{-1}} \! \regpredrisk{\pengradflowest{t}}
\le  43.7 (1+\overline{C}_{\lambda})\!\inf_{\tilde\lambda\in [\lambda,\lambda+2\norm{\Sigmahatlambda}]} \! \regpredrisk{\ridgeest[\tilde\lambda]}.
\end{equation*}
\end{corollary}

\begin{proof}
The first inequality follows directly from \cref{ThmMain}, upon replacing the infimum for CG over $\{\tau_t\,|\,t\ge (2\norm{\Sigmahatlambda})^{-1}\}$ by the infimum over $\R^+$  and noting $4/(1-e^{-1/2})^2\le 25.9$. The second inequality is derived from \cref{PropGFRR} with evaluation of the numerical constant.
\end{proof}

As a second application we establish that for sufficiently large penalties $\lambda$ the risk of the gradient flow estimator decays monotonically in $t$ to the risk of $\ridgeest$. While it is not clear whether the CG risk decays monotonically as well because the residual polynomial $\filterCG$ is data-dependent, \cref{ThmMain} then yields a monotone upper bound.

\begin{proposition}\label{PropGFMon}
Let $s_1\ge\ldots\ge s_p$ denote the eigenvalues of $\Sigmahat$ and $v_1,\ldots,v_p$  corresponding normalised eigenvectors. Then
for the target vector $\gamma=\beta_0$ and all $\lambda > 0$ satisfying
\begin{equation}\label{EqLambda}
\lambda \sum_{j=i}^p s_j\scapro{\beta_0}{v_j}^2 \ge \tfrac{\sigma^2}{n} \sum_{j=i}^p s_j \ \forall \, i=1,\dots,p,
\end{equation}
the regularisation path of gradient flow has monotonously decreasing prediction risk $t\mapsto \penpredrisk{\pengradflowest{t}}$ for all $t\ge 0$.
\end{proposition}

\begin{remark}
Condition~\eqref{EqLambda} on $\lambda$ is simple, but certainly not optimal in specific situations. When $\beta_0$ is in general position with respect to the basis of eigenvectors $v_i$ (recall that they are random), then $\scapro{\beta_0}{v_j}^2\thicksim p^{-1}\norm{\beta_0}^2$ holds and we require $\lambda\gtrsim \frac{\sigma^2p}{n\norm{\beta_0}^2}$, that is a penalty scaling at least with the inverse of the signal-to-noise ratio. It is natural that in this case stopping the iterations early usually does not decrease the stochastic error sufficiently to allow for smaller total prediction error. This case with $\tau=\norm{\beta_0}$ can also be generated genuinely in high dimensions by a random effects assumption $\beta_0\sim N(0,\tau^2p^{-1}I_p)$ on the parameter vector $\beta_0$, cf.\ \citet{DobWag2018}. Note that there exists a $\lambda>0$ satisfying Condition~\eqref{EqLambda} as long as there is no index $i\in \{1,\dots,p\}$ such that $s_i>0$ and $\scapro{\beta_0}{v_j}=0$ for all $j\ge i$. 

For the target vector $\gamma=\beta_\lambda$ in the excess prediction risk, the approximation error is asymptotically as $t\to\infty$ always of smaller order than the stochastic error and there is no monotonicity in this case.
\end{remark}

\begin{proof}[Proof of \cref{PropGFMon}]
We have $\gamma-\beta_\lambda=\lambda\Sigmahatlambda^{-1}\beta_0$ for $\gamma=\beta_0$.
By the formulas from \cref{PropErrDecomp} for $\pengradflowest{t}$  we obtain  for $t\ge 0$
\begin{equation*}
\tfrac{d}{dt}A_{\lambda,\gamma}(\filterGF)=-2\scapro{\exp(-t\Sigmahatlambda)\beta_\lambda+\lambda\Sigmahatlambda^{-1}\beta_0}
{\Sigmahatlambda^2\exp(-t\Sigmahatlambda)\beta_\lambda}
\le -2\lambda\scapro{\exp(-t\Sigmahatlambda)\beta_0}
{\Sigmahat\beta_0}
\end{equation*}
and similarly for the conditional expectation of the stochastic error
\begin{align*}
\tfrac{d}{dt} \E[S_{\lambda,\gamma}(\filterGF) \, | \, X]&=\tfrac{2\sigma^2}{n}\trace\big((I_p-\exp(-t\Sigmahatlambda))\exp(-t\Sigmahatlambda)\Sigmahat\big)\\
&\le \tfrac{2\sigma^2}{n}\trace\big(\exp(-t\Sigmahatlambda)\Sigmahat\big).
\end{align*}
Combining the bounds, we arrive at
\begin{equation*}
\tfrac{d}{dt}\penpredrisk{\pengradflowest{t}}\le -2e^{-t\lambda}\trace\Big(\exp(-t\Sigmahat)\Sigmahat\big(\lambda\beta_0\beta_0^\top - \tfrac{\sigma^2}{n}I_p
\big)\Big)
\end{equation*}
and we must check whether the trace is non-negative.
Expressing it in the basis $(v_i)$ and summation by parts yield, setting formally $e^{-ts_0} \coloneqq 0$,
\begin{equation*}
\trace\Big(\exp(-t\Sigmahat)\Sigmahat\big(\lambda\beta_0\beta_0^\top - \tfrac{\sigma^2}{n}I_p
\big)\Big)=\sum_{i=1}^p \big(e^{-ts_i}-e^{-ts_{i-1}}\big)\sum_{j=i}^p s_j\big(\lambda\scapro{\beta_0}{v_j}^2-\tfrac{\sigma^2}{n}\big).
\end{equation*}
A simple sufficient condition for this to be non-negative is that $\sum_{j=i}^p s_j(\lambda\scapro{\beta_0}{v_j}^2-\tfrac{\sigma^2}{n})\ge 0$ holds for all $i=1,\ldots,p$. Rearranging terms gives the result.
\end{proof}

Finally, we transfer the in-sample results to the out-of-sample setting. We assume that $\E[x_ix_i^\top]=\Sigma$ for all $i\in \{1, \dots, n\}$ and that $\Sigma$ is positive semidefinite. Similar to $\Sigmahatlambda$ we write for the population counterpart $\Sigmalambda \coloneqq \Sigma+\lambda I_p$. In analogy to \eqref{EqLin}, \eqref{EqRin}, let us consider the
 \emph{regularised out-of-sample prediction loss} and {\it risk}
\begin{equation*}
    \outregpredloss{\hat{\beta}} \coloneqq \norm{\Sigmalambda^{1/2}(\hat{\beta}-\gamma)}^2,\quad \outregpredrisk{\hat{\beta}} \coloneqq \E[\outregpredloss{\hat{\beta}}\,|\,X]
\end{equation*}
of an estimator $\hat\beta$ for a deterministic target vector $\gamma\in\R^p$. For the transfer from in-sample to out-of-sample risk we follow classical routes.

\begin{proposition}\label{PropOutPredRisk}
    Suppose that the normalised feature vectors $\Sigma^{-1/2}x_i$ (with $\Sigma^{-1/2}=(\Sigma^+)^{1/2}$ and the generalised inverse $\Sigma^+$) are $L$-sub-Gaussian:
    \begin{equation*}
    \textstyle\exists L>0\, \forall v\in \R^p:\; \sup_{m \ge 1} m^{-1/2} \E^{1/m}[\abs{\scapro{v}{\Sigma^{-1/2}x_i}}^m] \le L \E^{1/2}[\scapro{v}{\Sigma^{-1/2} x_i}^2].
    \end{equation*}
    Let $\mathcal{N}(\lambda) \coloneqq \trace(\Sigmalambda^{-1}\Sigma)$ denote the effective rank of $\Sigma$. Then there is a constant $C=C(L)>0$ such that for all $z \ge 1$ with probability at least $1-e^{-z}$
    \begin{equation}\label{LemOutPredRisk}
        \big|\outregpredloss{\hat{\beta}}-\regpredloss{\hat{\beta}}\big| \le C \Big( \sqrt{\tfrac{\mathcal{N}(\lambda)}{n}} \vee \tfrac{\mathcal{N}(\lambda)}{n} \vee \sqrt{\tfrac{z}{n}} \vee \tfrac{z}{n} \Big) \outregpredloss{\hat{\beta}}
    \end{equation}
    and
    \begin{equation*}
        \big|\outregpredrisk{\hat{\beta}}-\regpredrisk{\hat{\beta}}\big| \le C \Big( \sqrt{\tfrac{\mathcal{N}(\lambda)}{n}} \vee \tfrac{\mathcal{N}(\lambda)}{n} \vee \sqrt{\tfrac{z}{n}} \vee \tfrac{z}{n} \Big) \outregpredrisk{\hat{\beta}}
    \end{equation*}
    hold for any $\R^p$-valued estimator $\hat{\beta}$.
\end{proposition}

\begin{proof}
    We have
    \begin{align}
        \abs{\outregpredloss{\hat{\beta}}-\regpredloss{\hat{\beta}}} &= \abs{\scapro{\hat{\beta}-\gamma}{(\Sigmalambda - \Sigmahatlambda)(\hat{\beta}-\gamma)}} \notag\\
        &= \abs{\scapro{\Sigmalambda^{1/2}(\hat{\beta}-\gamma)}{\Sigmalambda^{-1/2}(\Sigma - \Sigmahat)\Sigmalambda^{-1/2}\Sigmalambda^{1/2}(\hat{\beta}-\gamma)}} \notag\\
        &\le \norm{\Sigmalambda^{-1/2}(\Sigma - \Sigmahat)\Sigmalambda^{-1/2}} \norm{\Sigmalambda^{1/2}(\hat{\beta}-\gamma)}^2 \notag\\
        &= \norm{\Sigmalambda^{-1/2}(\Sigma - \Sigmahat)\Sigmalambda^{-1/2}} \outregpredloss{\hat{\beta}}. \label{EqBoundOutPredLoss}
    \end{align}
    The transformed data vectors $x_i'=\Sigmalambda^{-1/2}x_i$, $i=1,\dots,n$, have the covariance matrix $\Sigma'= \Sigmalambda^{-1/2} \Sigma \Sigmalambda^{-1/2}$ and are by definition still sub-Gaussian with the same constant $L$.  For their sample covariance matrix $\Sigmahat' = \Sigmalambda^{-1/2} \Sigmahat \Sigmalambda^{-1/2}$ we deduce from \citet[Theorem~9]{KolLou2017} and $\mathcal{N}(\lambda)=\trace(\Sigma')$ that there is a constant $C=C(L) > 0$ such that for all $z \ge 1$ with probability at least $1-e^{-z}$
    \begin{equation}\label{EqKolLou}
        \norm{\Sigmalambda^{-1/2}(\Sigma - \Sigmahat)\Sigmalambda^{-1/2}} = \norm{\Sigma'-\Sigmahat'} \le C \Big( \sqrt{\tfrac{\mathcal{N}(\lambda)}{n}} \vee \tfrac{\mathcal{N}(\lambda)}{n} \vee \sqrt{\tfrac{z}{n}} \vee \tfrac{z}{n} \Big).
    \end{equation}
    This gives the first claim. Jensen's inequality then implies directly the second statement.
\end{proof}

The effective rank clearly satisfies $\mathcal{N}(\lambda)\le p (1+\lambda\norm{\Sigmahat})^{-1}$, but is usually much smaller. The bound in \cref{LemOutPredRisk} depends on $\outregpredloss{\hat{\beta}}$. By the argument in \citet[Lemma~5]{HucWah2023}, the following holds  for some $c_1,c_2>0$: If $\mathcal{N}(\lambda) \le c_1 n$, then $\norm{\Sigmalambda^{-1/2}(\Sigma - \Sigmahat)\Sigmalambda^{-1/2}} \le 1/2$ with probability at least $1-e^{-c_2n}$. Insertion into \eqref{EqBoundOutPredLoss} yields $\outregpredloss{\hat{\beta}} \le 2 \regpredloss{\hat{\beta}}$ and $\outregpredrisk{\hat{\beta}} \le 2 \regpredrisk{\hat{\beta}}$ with probability at least $1-e^{-c_2n}$. This allows us to control directly the out-of-sample prediction risk in terms of the in-sample risk. In conclusion, we can bound the out-of-sample prediction risk for CG iterations by the corresponding GF and RR risks in \cref{ThmMain} with high probability. We abstain from restating it in all technical detail. A different, elegant way for bounding the out-of-sample prediction risk has been found by \citet{MouRos2022}, but it seems tailored to ridge estimators. Let us  point out that spectral norm bounds like~\eqref{EqKolLou} also allow to derive bounds on $\overline{C}_{\lambda}$ directly in terms of the population feature covariance matrix $\Sigma$ with high probability.

\section{Numerical illustration}\label{SecNum}

\begin{figure}
    \centering
     \includegraphics[width=\textwidth]{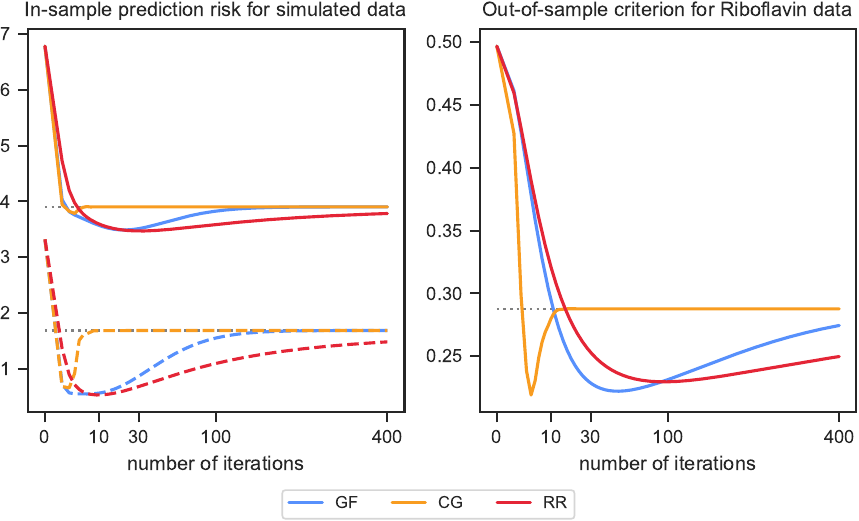}
    \caption{Left: Prediction risk in simulation example for $\gamma=\beta_0$ (solid lines) and $\gamma=\beta_\lambda$ (dashed lines). The grey lines represent the terminal risk of the corresponding ridge estimator $\ridgeest$. Right: Out-of-sample ridge criterion for the Riboflavin data set and the three estimation methods.}
    \label{FigSim1}
\end{figure}

In a simulation example, we consider a high-dimensional setting with $n=400$ observations and parameter dimension $p=500$. The features $x_i$ are i.i.d.\ $N(0,\Sigma)$-distributed where the eigenvalues of $\Sigma$ are $\lambda_1=\cdots=\lambda_{20}=100$ and $\lambda_{21}=\cdots=\lambda_{500}=1$. The error variables $\eps_i$ are i.i.d.\ $N(0,6)$-distributed. The coefficient vector $\beta_0\in\R^p$ was generated by an $N(0,p^{-1} I_p)$-law once for all $1000$ Monte Carlo runs. The penalisation parameter is chosen as $\lambda = 3$.
As a proxy for the gradient flow estimator $\pengradflowest{t}$ we consider the gradient descent iterates $\ridgegdest{k}$ with learning rate $\eta=1/(2\lambda + \norm{\Sigmahat})$. In \cref{FigSim1}(left) we plot the Monte Carlo  prediction risks of $\ridgecgest{k}$, $\ridgegdest{k}$ and $\ridgeest[\lambda+1/(\eta k)]$ as a function of the iteration number $k$, interpolating linearly in-between. The plot shows the in-sample prediction risk, but it turns out to be practically identical with the out-of-sample prediction risk.
The dashed lines indicate the results for excess prediction risk with $\gamma=\beta_\lambda$ and the solid lines refer to the penalised prediction risk with $\gamma=\beta_0$. The iteration numbers are plotted on a quadratic scale so that smaller iteration numbers have a better resolution.

Apparently, the regularisation paths all have a similar shape and there are intermediate iterates where the prediction risk is smaller than for the terminal value $\ridgeest$. This means that stopping (conjugate) gradient descent early does not only reduce computational time but also decreases statistical errors. The same effect can be produced by using a larger penalty $\lambda'>\lambda$ in ridge regression, but in practice the ridge estimator has again to be determined iteratively.

In this example we see that the minimal risk of conjugate gradients is slightly larger than the minimal risk of gradient descent and ridge regression, respectively.
On the other hand, conjugate gradients need much fewer iteration steps to attain the minimal risk value and to reach the terminal value than gradient descent. This computational advantage of CG is even much more pronounced in other examples where less regularisation is required, compare also the numerical results for polynomially decaying eigenvalues of $\Sigma$ in statistical inverse problems \citep{BlaHofRei2018,HucRei2025}. Despite the seemingly restrictive simulation setup, unreported simulations under more general conditions show comparable outcomes.

For illustration with real data, we consider the \textit{Riboflavin} data set from the R package hdi~\citep{MeiDezMei2021}. This data set contains data about the riboflavin production by Bacillus subtilis containing $n=71$ observations of $4088$ predictors (gene expressions) and a one-dimensional response (riboflavin production). After standardising the variables, we select a subset of $p=2n$ features once at random, leading to a high-dimensional setting with $p/n = 2$.
Then the data was split in $50$ training observations and $21$ test data points, drawn randomly 1000 times. For the penalisation parameter $\lambda =0.1$, \cref{FigSim1}(right) shows the regularisation paths in terms of the out-of-sample criterion ${\cal E}_\lambda^{\mathrm{out}}(\hat\beta)$, the ridge criterion ${\cal E}_\lambda(\hat\beta)$ evaluated on the test set, taking the mean over all $1000$ splits in training and test data. Again, the minima of the different procedures are comparable, this time with CG performing best. CG needs significantly less iteration steps to attain the minimal risk as well as to converge to the ridge estimator. For $\lambda=0$ the curves look similar, but with a more pronounced U-shape.
In more general settings, unreported simulations with heteroscedastic errors and sparse signals yield risk curves of the same shape as those shown in \cref{FigSim1}(left).

\section{Conclusion}\label{SecConclusion}

Building on the error decomposition \eqref{EqCGDecomp}, we have shown that the risk for the iterates of conjugate gradients can be bounded by that for the iterates of standard gradient flow up to a factor depending only on the penalised empirical covariance matrix $\Sigmahatlambda$. This comparison involves the time shift $t\mapsto \tau_t$, the inverse function of  $\tilde t\mapsto \abs{(\filterCG[\tilde{t}])'(0)}/2$, which can be computed along the CG iterations. While the results in \cite{HucRei2025} have paved the way for this analysis, we have been able to transfer them from the \emph{CG algorithm applied to the normal equations} (CGNE) in inverse problems to the direct CG algorithm in ridge regression. Moreover, we can cover two classes of general loss functions $\ell^{\mathrm{in}}_{\lambda,\gamma}$ and $\ell^{\mathrm{out}}_{\lambda,\gamma}$ for the penalised least squares criterion, and we can bound the factor $\overline{C}_{\lambda}$ for covariance matrix specifications of statistical interest. Since the Krylov subspaces are generated by the penalised quantities $\Sigmahatlambda$ and $y_{\lambda}$, the underlying geometry of the CG analysis, however, does not seem to allow for sufficiently tight bounds for the pure (unpenalised) prediction loss, where $\lambda=0$ and $\gamma = \beta_0$. The simulations exhibit indeed a very similar shape of the regularisation paths and demonstrate that CG achieves its minimal loss and the loss of the terminal ridge estimator much faster than GF.

A natural question is whether we can bound the risks of the regularisation paths in the converse direction of \cref{CorOracleRisk}. Since the bias of gradient flow decays exponentially in $t$, while the bias of ridge regression only decays polynomially, it is clear that for noise level $\sigma\downarrow 0$ we have $\inf_{t\ge 0}\regpredrisk{\ridgeest[\lambda+1/t]}/\inf_{t\ge 0}\regpredrisk{\pengradflowest{t}}\to\infty$, see also the behaviour of the filter functions in \cref{FigResidualPolynomial}(left). For the comparison of gradient flow with conjugate gradients and $t \ge (2\norm{\Sigmahatlambda})^{-1}$ a rigorous analysis is difficult, but we do also not expect that the risk of gradient flow can be bounded universally by that of conjugate gradients. Conjugate gradients seek a solution in the $y$-dependent Krylov subspaces while the regularisation of gradient flow depends only on the empirical covariance of the feature vectors $x_i$. Based on this insight, \citet{FinKri2023} demonstrate convincingly that for certain latent factor models conjugate gradients achieve a small risk, while this cannot be expected for unsupervised (only feature-dependent) regularisation.

In summary, our results demonstrate that conjugate gradients do not only provide a fast numerical solver, but also allow for a clear statistical analysis and come with practically the same statistical guarantees as standard linear methods like GF, GD and RR.  Profiting from the regularisation power, model selection and early stopping rules for CG iterates seem very attractive (compare \citet{HucRei2025} for the case $\lambda=0$) and a closer investigation of the benign overfitting phenomenon for gradient methods would be fascinating, in particular concerning the eigenvalue conditions in \citet{TsiBar2023} and ours in \cref{ThmMain}.


\paragraph{Acknowledgement.}

Financial support by the Deutsche Forschungsgemeinschaft (DFG) and the Austrian Science Fund (FWF, I 5484-N) through the research unit FOR 5381 \emph{Mathematical Statistics in the Information Age -- Statistical Efficiency and Computational Tractability} and by the Aarhus University Research Foundation (AUFF, 47221 and 47388) is gratefully acknowledged.


\bibliography{bibliography}

\begin{thebibliography}{24}
\providecommand{\natexlab}[1]{#1}
\providecommand{\url}[1]{\texttt{#1}}
\providecommand{\urlprefix}{URL }
\providecommand{\eprint}[2][]{\url{#2}}

\bibitem[{Ali, Dobriban and Tibshirani(2020)}]{AliDobTib2020}
Ali, A., Dobriban, E., Tibshirani, R.J.: The implicit regularization of
  stochastic gradient flow for least squares.
\newblock In: Proceedings of the 37th International Conference on Machine
  Learning. pp. 233--244. ICML'20 (2020)

\bibitem[{Ali, Kolter and Tibshirani(2019)}]{AliKolTib2019}
Ali, A., Kolter, J.Z., Tibshirani, R.J.: A continuous-time view of early
  stopping for least squares regression.
\newblock In: Chaudhuri, K., Sugiyama, M. (eds.) Proceedings of the
  Twenty-Second International Conference on Artificial Intelligence and
  Statistics. Proceedings of Machine Learning Research, vol.~89, pp.
  1370--1378. PMLR (2019).
\newblock \url{https://proceedings.mlr.press/v89/ali19a/ali19a.pdf}

\bibitem[{{Blanchard}, {Hoffmann} and {Rei\ss{}}(2018)}]{BlaHofRei2018}
{Blanchard}, G., {Hoffmann}, M., {Rei\ss{}}, M.: Optimal adaptation for early
  stopping in statistical inverse problems.
\newblock SIAM/ASA J.~Uncertain.~Quantif. \textbf{6}(3), 1043--1075 (2018).
\newblock \url{https://doi.org/10.1137/17m1154096}

\bibitem[{{Blanchard} and {Kr\"{a}mer}(2016)}]{BlaKra2016}
{Blanchard}, G., {Kr\"{a}mer}, N.: Convergence rates of kernel conjugate
  gradient for random design regression.
\newblock Anal.~Appl. \textbf{14}(06), 763--794 (2016).
\newblock \url{https://doi.org/10.1142/s0219530516400017}

\bibitem[{{B\"{u}hlmann} and {van de Geer}(2011)}]{BueGee2011}
{B\"{u}hlmann}, P., {van de Geer}, S.: Statistics for High-Dimensional Data:
  Methods, Theory and Applications.
\newblock Springer Series in Statistics. Springer, Berlin, Heidelberg (2011).
\newblock \url{https://doi.org/10.1007/978-3-642-20192-9}

\bibitem[{{Chen} et~al.(2025){Chen}, {Huber}, {Lin} and {Zaid}}]{CheHubLin2025}
{Chen}, T., {Huber}, C., {Lin}, E., {Zaid}, H.: Preconditioning without a
  preconditioner: faster ridge-regression and {G}aussian sampling with
  randomized block {K}rylov subspace methods  (2025).
\newblock \url{https://arxiv.org/abs/2501.18717}

\bibitem[{{Dobriban} and {Wager}(2018)}]{DobWag2018}
{Dobriban}, E., {Wager}, S.: High-dimensional asymptotics of prediction: Ridge
  regression and classification.
\newblock Ann.~Stat. \textbf{46}(1), 247--279 (2018).
\newblock \url{https://doi.org/10.1214/17-aos1549}

\bibitem[{{Engl}, {Hanke} and {Neubauer}(1996)}]{EngHanNeu1996}
{Engl}, H.W., {Hanke}, M., {Neubauer}, A.: Regularization of Inverse Problems.
\newblock Mathematics and Its Applications, vol. 375. Kluwer Academic
  Publishers, Dordrecht (1996)

\bibitem[{{Finocchio} and {Krivobokova}(2023)}]{FinKri2023}
{Finocchio}, G., {Krivobokova}, T.: An extended latent factor framework for
  ill-posed linear regression  (2023).
\newblock \url{https://arxiv.org/abs/2307.08377}

\bibitem[{{Hastie}, {Tibshirani} and {Friedman}(2009)}]{HasTibFri2009}
{Hastie}, T., {Tibshirani}, R., {Friedman}, J.: The Elements of Statistical
  Learning, 2nd edn.
\newblock Springer Series in Statistics. Springer, New York (2009).
\newblock \url{https://doi.org/10.1007/978-0-387-84858-7}

\bibitem[{Helland(1990)}]{Hel1990}
Helland, I.S.: Partial least squares regression and statistical models.
\newblock Scand.~J.~Stat. \textbf{17}(2), 97--114 (1990).
\newblock \url{http://www.jstor.org/stable/4616159}

\bibitem[{{Hucker} and {Rei\ss{}}(2025)}]{HucRei2025}
{Hucker}, L., {Rei\ss{}}, M.: Early stopping for conjugate gradients in
  statistical inverse problems.
\newblock Numer.~Math. \textbf{157}(5), 1739--1791 (2025).
\newblock \url{https://doi.org/10.1007/s00211-025-01469-4}

\bibitem[{{Hucker} and {Wahl}(2023)}]{HucWah2023}
{Hucker}, L., {Wahl}, M.: A note on the prediction error of principal component
  regression in high dimensions.
\newblock Theory Probab.~Math.~Stat. \textbf{109}(0), 37--53 (2023).
\newblock \url{https://doi.org/10.1090/tpms/1196}

\bibitem[{{Koltchinskii} and {Lounici}(2017)}]{KolLou2017}
{Koltchinskii}, V., {Lounici}, K.: Concentration inequalities and moment bounds
  for sample covariance operators.
\newblock Bernoulli \textbf{23}(1), 110--133 (2017).
\newblock \url{https://doi.org/10.3150/15-bej730}

\bibitem[{{Meier} et~al.(2021){Meier}, {Dezeure}, {Meinshausen}, {Maechler} and
  {Buehlmann}}]{MeiDezMei2021}
{Meier}, L., {Dezeure}, R., {Meinshausen}, N., {Maechler}, M., {Buehlmann}, P.:
  hdi: High-dimensional inference  (2021).
\newblock \url{https://CRAN.R-project.org/package=hdi}

\bibitem[{{Mourtada} and {Rosasco}(2022)}]{MouRos2022}
{Mourtada}, J., {Rosasco}, L.: An elementary analysis of ridge regression with
  random design.
\newblock C.~R.~Math. \textbf{360}(G9), 1055--1063 (2022).
\newblock \url{https://doi.org/10.5802/crmath.367}

\bibitem[{{Nazareth}(1979)}]{Naz1979}
{Nazareth}, L.: A relationship between the {BFGS} and conjugate gradient
  algorithms and its implications for new algorithms.
\newblock SIAM~J.~Numer.~Anal. \textbf{16}(5), 794--800 (1979).
\newblock \url{https://doi.org/10.1137/0716059}

\bibitem[{Nemirovskii(1986)}]{Nem1986}
Nemirovskii, A.S.: The regularizing properties of the adjoint gradient method
  in ill-posed problems.
\newblock USSR Comput.~Math.~Math.~Phys. \textbf{26}(2), 7--16 (1986).
\newblock \url{https://doi.org/10.1016/0041-5553(86)90002-9}

\bibitem[{Pedregosa et~al.(2011)Pedregosa, Varoquaux, Gramfort, Michel,
  Thirion, Grisel, Blondel, Prettenhofer, Weiss, Dubourg, Vanderplas, Passos,
  Cournapeau, Brucher, Perrot and {{\'E}}douard Duchesnay}]{PedVarGra2011}
Pedregosa, F., Varoquaux, G., Gramfort, A., Michel, V., Thirion, B., Grisel,
  O., Blondel, M., Prettenhofer, P., Weiss, R., Dubourg, V., Vanderplas, J.,
  Passos, A., Cournapeau, D., Brucher, M., Perrot, M., {{\'E}}douard Duchesnay:
  Scikit-learn: Machine learning in {P}ython.
\newblock J.~Mach.~Learn.~Res. \textbf{12}(85), 2825--2830 (2011).
\newblock \url{http://jmlr.org/papers/v12/pedregosa11a.html}

\bibitem[{{Phatak} and {de Hoog}(2002)}]{PhaHoo2002}
{Phatak}, A., {de Hoog}, F.: Exploiting the connection between {PLS}, {L}anczos
  methods and conjugate gradients: alternative proofs of some properties of
  {PLS}.
\newblock J.~Chemom. \textbf{16}(7), 361--367 (2002).
\newblock \url{https://doi.org/10.1002/cem.728}

\bibitem[{{Rosipal} and {Kr\"{a}mer}(2006)}]{RosKra2006}
{Rosipal}, R., {Kr\"{a}mer}, N.: Overview and Recent Advances in Partial Least
  Squares, pp. 34--51.
\newblock Springer, Berlin, Heidelberg (2006).
\newblock \url{https://doi.org/10.1007/11752790\_2}

\bibitem[{{Singer} et~al.(2016){Singer}, {Krivobokova}, {Munk} and {de
  Groot}}]{SinKriMun2016}
{Singer}, M., {Krivobokova}, T., {Munk}, A., {de Groot}, B.: Partial least
  squares for dependent data.
\newblock Biometrika \textbf{103}(2), 351--362 (2016).
\newblock \url{https://doi.org/10.1093/biomet/asw010}

\bibitem[{Tsigler and Bartlett(2023)}]{TsiBar2023}
Tsigler, A., Bartlett, P.L.: Benign overfitting in ridge regression.
\newblock J.~Mach.~Learn.~Res. \textbf{24}(123), 1--76 (2023).
\newblock \url{http://jmlr.org/papers/v24/22-1398.html}

\bibitem[{{Wu} et~al.(2025){Wu}, {Bartlett}, {Telgarsky} and
  {Yu}}]{WuBarTel2025}
{Wu}, J., {Bartlett}, P., {Telgarsky}, M., {Yu}, B.: Benefits of early stopping
  in gradient descent for overparameterized logistic regression  (2025).
\newblock \url{https://arxiv.org/abs/2502.13283}

\end{thebibliography}


\newpage
\begin{appendix}
\section*{Summary of notation}\label{TblNotation}
\vfill
\begin{center}
\begin{turn}{90}
\renewcommand{\arraystretch}{1.3}
\setlength{\tabcolsep}{6pt}
\scalebox{0.87}{%
\begin{tabular}{@{}llll@{}}
\toprule
 & \textbf{Ridge regression} \; \eqref{EqDefRRfilter} & \textbf{Gradient flow \; \eqref{EqDefPenGF}} & \textbf{Conjugate gradients} \; \eqref{EqDefInterpolCG}, \eqref{EqDefInterpolResPol}, \eqref{EqDefTruncatedResPols} \\
\midrule
Estimator
  & $\ridgeest[\lambda'] = \Sigmahatlambda^{-1/2}(I_p-\filterRR(\Sigmahatlambda)) y_{\lambda}$
  & $\pengradflowest{t} = \Sigmahatlambda^{-1/2}(I_p-\filterGF(\Sigmahatlambda)) y_{\lambda}$
  & $\ridgecgest{t} = \Sigmahatlambda^{-1/2} (\id_p - \filterCG(\Sigmahatlambda))y_{\lambda}$ \\[6pt]

Filter function
  & $\filterRR(x)=\tfrac{\lambda'-\lambda}{\lambda'-\lambda+x}$
  & $\filterGF(x)=\exp(-tx)$
  &
  $\filterCG[k+\alpha](x) = (1-\alpha)\filterCG[k](x) + \alpha \filterCG[k+1](x),$\\
  & & & $k=0,\dots,\tilde{p}-1, \ \alpha \in (0,1],$\\
  & & & $\filterCG[t,<](x) = \filterCG(x){\bf 1}(x<x_{1,t}), \ t \in [0,\tilde{p}],$\\
  & & & $x_{1,t}$ smallest zero of $\filterCG$ \\[6pt]

Error decomposition
  & \multicolumn{2}{l}{$\regpredloss{\hat\beta} = A_{\lambda,\gamma}(R) + S_{\lambda}(R) - 2C_{\lambda,\gamma}(R)$}
  & $\ridgepredloss{\ridgecgest{t}} = \ridgeapproxerrorCG{t} + \ridgestocherrorCG{t} - 2\ridgecrossCG{t}$ \\
  & \multicolumn{2}{l}{for $R \in \{\filterRR, \filterGF\}$ \; (Prop.~\ref{PropErrDecomp})} & $\hphantom{\ridgepredloss{\ridgecgest{t}}} \le 2\ridgeapproxerrorCG{t} + 2\ridgestocherrorCG{t}$ \;(Prop.~\ref{PropErrorDecomps}) \\[6pt]

Approximation error
  & \multicolumn{2}{l}{$A_{\lambda,\gamma}(R)=\norm{\Sigmahatlambda^{1/2}(R(\Sigmahatlambda)\beta_\lambda+(\gamma-\beta_\lambda))}^2$}
  & $\ridgeapproxerrorCG{t} \le \norm{\Sigmahatlambda^{1/2}\filterCG[t,<](\Sigmahatlambda)^{1/2}\beta_{\lambda}}^2$ \\[6pt]

Stochastic error
  & \multicolumn{2}{l}{$S_{\lambda}(R)=\norm{(I_p-R(\Sigmahatlambda))\varepsilon_\lambda}^2$}
  & $\ridgestocherrorCG{t} = \norm{(I_p-\filterCG[t,<](\Sigmahatlambda))^{1/2}\varepsilon_{\lambda}}^2$ \\
\bottomrule
\end{tabular}}
\end{turn}
\end{center}
\end{appendix}
\vfill
\newpage


\end{document}